\documentclass[11pt, reqno]{amsart}
\usepackage[T1]{fontenc}
\usepackage{graphicx}
\usepackage{xcolor}
\usepackage{booktabs}
\usepackage[misc]{ifsym}
\usepackage{amssymb}
 \usepackage[all]{xy}
\usepackage[misc]{ifsym}
\usepackage{ifdraft}
\usepackage{amsmath,bm,xfrac,mathabx,mathtools,stackengine}
\usepackage{algorithm}
\usepackage{algpseudocode}
\usepackage[colorlinks=true, pdfstartview=FitV, linkcolor=blue, citecolor=blue, urlcolor=blue]{hyperref}
\usepackage[justification=centering]{caption}
\usepackage{comment}

\usepackage{xr}
\usepackage{mwe}

\newcommand{\delete}[1]{}
\newcommand{\on}[1]{\operatorname{#1}}

\title[Provably Accurate Adaptive Sampling for Collocation Points in PINNs]{Provably Accurate Adaptive Sampling for Collocation Points in Physics-informed Neural Networks}
\author[A. CARADOT]{Antoine CARADOT$^1$}
\author[R. EMONET]{R\'emi Emonet$^1$}
\author[A. HABRARD]{Amaury Habrard$^1$}
\author[A.R. Mezidi]{Abdel-Rahim Mezidi$^1$}
\author[M. SEBBAN]{Marc Sebban$^1$}

\address{$^1$Hubert Curien Laboratory, Jean Monnet University, Saint-\'Etienne, France}
\date{\today}     

\keywords{PINN ; Collocation points ; Adaptive sampling ; Quadrature method.}
\thanks{2020 {\it Mathematics Subject Classification:}
Primary 68T07, 41A55; Secondary 62D05, 35A25}

\theoremstyle{plain}
\newtheorem{theorem}{Theorem}[section]
\newtheorem{lemma}{Lemma}[section]
\newtheorem{proposition}{Proposition}[section]

\theoremstyle{remark}

\begin{document}

\begin{abstract}
Despite considerable scientific advances in numerical simulation, efficiently solving PDEs remains a complex and often expensive problem. Physics-informed Neural Networks (PINN) have emerged as an efficient way to learn surrogate solvers by embedding the PDE in the loss function and minimizing its residuals using automatic differentiation at so-called collocation points. Originally uniformly sampled, the choice of the latter has been the subject of recent advances leading to adaptive sampling refinements for PINNs. In this paper, leveraging a new quadrature method for approximating definite integrals, we introduce a provably accurate sampling method for collocation points based on the Hessian of the PDE residuals. Comparative experiments conducted on a set of 1D and 2D PDEs demonstrate the benefits of our method.
\end{abstract}

\maketitle              

\delete{
The resolution of systems of partial differential equations often requires the involvement of numerical methods, as analytical solutions are rarely obtainable in practice. In recent years, the use of neural networks in the search for solutions has been successfully implemented. This approach consists in embedding the constraints of the differential system within the loss function of the network. During the training phase, the network is evaluated on a set of sampled collocation points, and the resulting loss function is used to update the parameters of the network. The loss function then decreases, and the output of the network  becomes closer to satisfying the differential system. In this paper, we present an adaptative sampling method for the collocation points based on the variations of the loss function. This approach allows for the number of sampled points to remain the same throughout the training, while the distribution of the points in the domain is tied to the derivative of the loss function, hence the name of derived adaptative sampling. We will first present a theoretical motivation for this sampling procedure, and then, through several examples, compare it with other classical sampling methods such as the random uniform sampling and RAR.
}

\section{Introduction}

Incorporating domain knowledge into machine learning algorithms has become a widespread strategy for managing ill-posed problems, data scarcity and solution consistency. Indeed, ignoring the fundamental principles of the underlying theory may lead to, yet optimal, implausible solutions yielding poor generalization and predictions with a high level of uncertainty. Embedding domain knowledge has been shown to be useful  when used at different levels of the learning process for (i) constraining/regularizing the optimization problem, (ii) designing suitable theory-guided loss functions, (iii) initializing  models with meaningful parameters, (iv) designing consistent neural network  architectures, or (v) building (theory/data)-driven hybrid models. In this context, physics is probably  the scientific domain that has benefited the most during the past years from  advances in the so-called {\it Physics-informed Machine Learning} (PiML) field \cite{karniadakis_physics-informed_2021} by leveraging physical laws, typically in the form of Partial Differential Equations (PDEs) that govern some underlying dynamical system. This new line of research led to a novel generation of deep-learning architectures, including Neural ODE \cite{chen2019neural}, PINN \cite{raissi2019physics}, FNO \cite{LiKALBSA21}, PINO \cite{li2023physicsinformed},  PDE-Net \cite{long2018pdenet}, etc.\\

In this paper, we specifically focus on Physics-informed Neural Networks (PINNs) that have received much attention from the PiML community and can be used for both forward as well as inverse problems for differential equations. Despite important scientific advances in numerical simulation, solving efficiently PDEs remains complex and often prohibitively costly. By embedding the physical knowledge into the loss function, PINNs appeared as a natural way for learning efficient neural PDE solvers by minimizing the residuals at collocation points typically randomly sampled from the spatio-temporal domain. Despite indisputable progress, PINNs are still  at an early stage and it has become crucial to study their theoretical foundations and algorithmic properties to gain a comprehensive grasp of their capabilities and limitation. Indeed, different studies have shown that PINNs may be subject to pathological behaviors, leading to trivial solutions with 0 residuals, thus plausible w.r.t. the physical law, while converging to an incorrect solution \cite{Chandrajit2021,doumeche2023convergence}. Characterizing these “failure modes” \cite{wang2020pinnsfailtrainneural} has led to an  active area of research addressing this task from two main perspectives: a first line of investigation that aims at building theoretical foundations when learning PINNs from a uniform sampling of   collocation points (e.g., equispaced uniform grid or uniformly random sampling), resulting in consistency and convergence guarantees in the form of estimation/approximation/optimization  bounds  (see, e.g., \cite{deryck2023error,deryck2022generic,doumeche2023convergence,girault:hal-04518335,lanthaler2022}); a second one with the objective of enhancing PINN performance through the lens of the collocation point sampling. Rather than drawing them uniformly, several intuitive strategies have flourished in the literature that suggest guiding the selection during the learning process according to the magnitude (or the gradient) of the PDE residuals. This gave rise to a new family of {\it adaptive sampling} methods for PINNs (see, e.g., \cite{pmlr-v202-daw23a,doi:10.1137/19M1274067,peng2022,subramanian2023,visser2024,Chenxi2022,gPINN2021}). However, it is worth noting that even though these methods have shown remarkable performances in practice, they share the common feature of not coming with theoretical guarantees of their advantage over a uniform sampling.

The objective of this paper is to bridge the gap by providing two new methodological contributions: (i) Recalling that minimizing an empirical loss in machine learning can be approached from a mathematical perspective as the approximation of the integral of some function $f$, we propose a new quadrature rule based on a simple trapezoidal interpolation and information about the second-order derivative $f{''}$. We derive an upper bound on the approximation error and show its tightness compared to that of issued from an equispaced uniform grid. This theoretical result is supported by several experiments. (ii) This finding prompts us to design a new theoretically founded adaptive sampling method for PINNs where $f$ takes the form of the residual-based loss function. This  strategy 
selects collocation points in the spatio-temporal domain where $f{''}$ varies the most. Experiments conducted on 1D and 2D PDEs highlight the interesting properties of our method.

The rest of this paper is organized as follows: in Section~\ref{sec:BG}, we introduce the  necessary background and related work; Section~\ref{sec:quadrature} is devoted to the presentation of our refined quadrature method and the upper bound derived on the total approximation error. In Section~\ref{sec:PINN}, we leverage our quadrature method to propose a new adaptive sampling method for PINNs and test it on 1D and 2D PDEs.

\section{Background and Related Work} \label{sec:BG}
In this paper, we consider PDEs of the general form: $\frac{\partial u}{\partial t}+{\mathcal N[}u;\phi]=0,$ 
where ${\mathcal N[}\cdot ;\phi]$ is a possibly nonlinear operator parameterized by $\phi$ and involving partial derivatives in either time or (multidimensional) space and $u(t,\mathbf{x})$ is the latent hidden solution, with $t \in [0,T]$ and  $\mathbf{x} \in \Omega$. This equation is typically augmented by appropriate initial and boundary conditions defined respectively as follows: 
\begin{eqnarray}
    {\mathcal I}[u](0,\mathbf{x})=0,&  & \hspace{0.5cm} \mathbf{x} \in \Omega \label{eq:IC} \nonumber \\
 {\mathcal B}[u](t,\mathbf{x})=0, & & \hspace{0.5cm} \mathbf{x} \in \partial \Omega, t \in [0,T]   \nonumber \label{eq:BC}
\end{eqnarray}
where ${\mathcal B}$ is a boundary operator that applies to the domain boundary $\partial \Omega$, and  ${\mathcal I}$ is an initial operator describing what happens at $t=0$.

A PINN \cite{raissi2019physics} aims at learning an approximation $u_{\theta}(t,\mathbf{x})$ of the solution $u(t,\mathbf{x})$ by optimizing the parameters $\theta$ of a neural network through the minimization of a loss ${\mathcal L}(\theta)$ composed of the following non-negative PDE residual terms:
\begin{eqnarray}
{\mathcal L}_{\mathcal N}(\theta) & = & \int_{[0,T] \times \Omega} \left(\frac{\partial u_{\theta}}{\partial t}+{\mathcal N[}u_{\theta};\phi]\right)^2 dtd\mathbf{x}  \label{eq:col} \nonumber \\
{\mathcal L}_{\mathcal I}(\theta) & = & \int_{\Omega} ({\mathcal I}[u_{\theta}](0,\mathbf{x}))^2 d\mathbf{x}  \label{eq:init} \nonumber \\
{\mathcal L}_{\mathcal B}(\theta) & = & \int_{\partial \Omega} \left( {\mathcal B}[u_{\theta}](t,\mathbf{x}) \right)^2 dtd\mathbf{x} \label{eq:bound} \nonumber 
\end{eqnarray}
Therefore, a PINN optimization problem takes the following form\footnote{Note that PINNs can easily incorporate both PDE information and data measurements into the loss function. Our contributions still hold in such hybrid scenario.}:
\begin{eqnarray}
\underset{\theta}{\min} \hspace{0.1cm} {\mathcal L}(\theta) = \underset{\theta}{\min} (  {\mathcal L}_{\mathcal N}(\theta)+\lambda_1{\mathcal L}_{\mathcal I}(\theta)+\lambda_2{\mathcal L}_{\mathcal B}(\theta)+\lambda_3R(\theta)), \label{eq:loss}
\end{eqnarray}
where $\lambda_1,\lambda_2,\lambda_3$ are hyperparameters and $R(\cdot)$ is some regularization term. Since ${\mathcal L}(\theta)$ involves integrals, it cannot be directly minimized. In practice, these three integrals are approximated by finite sums computed over $N_{\mathcal N}$ {\it collocation}, $N_{\mathcal I}$ {\it initial} and $N_{\mathcal B}$ {\it boundary} points respectively. \\

From a mathematical perspective, one of the underlying problems when solving Eq.\eqref{eq:loss} involves approximating the integral of some function $f:\Omega\longrightarrow \mathbb{R}$ from $N$ evaluations of the integrand by a suitable {\it numerical quadrature rule} such that:
\begin{eqnarray}
     \sum_{i=1}^N w_if(\mathbf{x}_i) \approx \int_{\Omega} f(\mathbf{x})d\mathbf{x}, \label{eq:quadra-rule}
\end{eqnarray}
where $ w_i\geq 0$ are so-called quadrature weights. It is well-known  that the accuracy of this approximation depends on the chosen quadrature rule, the regularity of $f$ and the number of quadrature points $N$. 
If this remark obviously holds for any machine learning problem minimizing the empirical counterpart of some {\it true risk} with $N$ training data, it is even truer when it comes to learning surrogate neural solvers of complicated PDEs. This explains why, despite a remarkable effectiveness, PINNs have been shown to face pathological behaviors. In particular, they can be subject to trivial solutions with 0 residuals while converging to an incorrect solution as illustrated, e.g., in \cite{Chandrajit2021,doumeche2023convergence} (characterized as “failure modes” of PINNs, see, e.g., \cite{wang2020pinnsfailtrainneural}). 
One way to overcome this pitfall consists in resorting to a suitable regularization term $R(\theta)$ (in Eq.\eqref{eq:loss}) as done in gPINN \cite{gPINN2021} that embeds the gradient of the PDE residuals in the loss so as to enforce their derivatives to be zero as well, or in \cite{doumeche2023convergence}, where the authors use a ridge regularization  associated with a Sobolev norm to make PINNs both consistent and strongly convergent. 

Regularization apart, the location and distribution of the $N=N_{\mathcal N}+N_{\mathcal I}+N_{\mathcal B}$ quadrature points are key and they can have a significant influence on the accuracy and/or the convergence of PINNs. Yet, equispaced uniform grids and uniformly random sampling have been widely used up to now and it is only recently that the placement of these quadrature points has become an active area of research for PINNs leading to several adaptive nonuniform sampling methods (see an extensive comparison study, e.g., in \cite{Chenxi2022}). Beyond being easy to operate, one reason that may justify the still widespread use of uniform sampling stems from the resulting possibility to leverage theoretical frameworks for deriving error estimates for PINNs. For instance,  using a midpoint quadrature rule with a regular grid has led to the first  approximation error bounds with tanh PINNs (see, e.g., \cite{De_Ryck_2021,girault:hal-04518335}). On the other hand, taking advantage of  uniformly sampled collocation points and resorting to concentration inequalities, the authors of \cite{doumeche2023convergence} derived generalization bounds for this new family of networks. Setting theoretical considerations aside, several methods have been designed during the past four years for experimentally improving  uniform sampling approaches. Residual-based Adaptive Refinement \cite{doi:10.1137/19M1274067} (a.k.a. RAR) is a greedy adaptive method which consists in adding new collocation points along the learning iterations by selecting the locations where the PDE residuals are the largest. 
Although RAR has been shown to improve the performance of PINNs, its main limitation (beyond the requirement of a dense set of  collocation candidates) lies in the fact that it reduces the opportunity to explore other regions of the space by always picking locations with the largest residuals. Introduced in 2023, 
RAD \cite{Chenxi2022}, for Residual-based Adaptive Distribution, replaces the current collocation points by new ones drawn according to a distribution proportional to the PDE residuals, thus introducing some stochasticity in the sampling process.  A hybrid method, called RAR-D, combines RAR and RAD by stacking new points according to the density function. Both RAD and RAR-D (and other adaptive residual-based distribution variants, e.g., \cite{pmlr-v202-daw23a,peng2022})  have been shown to perform better than non-adaptive uniform sampling \cite{Chenxi2022}. In this category of methods, Retain-Resample-Release sampling (R3) algorithm \cite{pmlr-v202-daw23a} is the only one that accumulates collocation points in regions of high PDE residuals and which comes with guarantees. Indeed, the authors prove that this algorithm retains points from high residual regions if they persist over iterations  and releases points if they have been resolved by PINN training.\\

While the previous methods leverage the magnitude of the PDE residuals to guide the selection of the locations, some others employ their gradient. This is the case in \cite{gPINN2021} where the authors combine gPINN and RAR. In the same vein, the authors of \cite{subramanian2023} present an {\it Adaptive Sampling for Self-supervision} method that allows a combination of uniformly sampled points and data drawn according to the residuals or their gradient. The first-order derivative has been also recently exploited in PACMANN \cite{visser2024} which leverages gradient information for moving collocation points toward regions of higher residuals using gradient-based optimization.  These methods have been shown to further improve the performance of  PINNs, especially for PDEs where solutions have steep changes.  Drawing inspiration from these derivative-based methods, we define in the next section  a new provably accurate quadrature rule for approximating  the integral of a function. We prove that this method based on a simple trapezoid-based interpolation and second-order derivative information gives a tighter error bound than an equispaced uniform grid-based quadrature. 
Leveraging this finding, we present then, as far as we know, the first theoretically founded adaptive sampling method of collocation points for PINNs based on the Hessian of the PDE residuals.

\section{Quadrature Rules} \label{sec:quadrature}
Let $a, b \in \mathbb{R}$ and consider a function $f:[a,b] \longrightarrow \mathbb{R}$. We recall that the goal of the quadrature problem is to approximate the integral $\int_a^bf(x)dx$ by an expression of the form $\sum_{i}w_if(x_i)$ where the $w_i \in \mathbb{R}$ are the weights of the quadrature points $x_i$. There are two main strategies of quadrature rules:
\begin{enumerate}
    \item We take $x_0, \dots, x_N \in [a,b]$ and the weights $w_0, \dots, w_N$ are obtained by approximating the function $f$ using polynomials. This is the Newton-Cotes method.
    \item We fix a scalar product on the space of polynomials, which provides an orthonormal basis via a Gram-Schmidt procedure. Then the zeros of one element of this basis are the $x_0, \dots, x_N$, and the weights are found by a matrix inversion. 
\end{enumerate}
The second approach, in which the points are a consequence of the chosen scalar product, is very effective and has many variants depending on the interval $[a,b]$, such as Gauss-Legendre for $a,b \in \mathbb{R}$, Gauss-Chebyshev for $[a,b]=[-1,1]$, and Gauss-Hermite for $[a,b]=[-\infty,\infty]$. However, as our objective is to leverage a quadrature rule for designing a new efficient adaptive sampling method for PINNs, the first approach appears to be much more suitable from a computational perspective because it does not require to compute the orthonormal basis as well as the zeros of one of its elements followed by matrix inversion.\\

On the other hand, an issue one might encounter using the Newton-Cotes method is that, as it relies on a polynomial approximation of $f$, the so-called Runge's phenomenon might occur. This problem happens when   $\on{max}_{x \in [a,b]}|f^{(n)}(x)|$ is an increasing function of $n$, where $f^{(n)}$ is the $n^{th}$-order derivative. Under these circumstances, the interpolation polynomial of $f$ may have sharp oscillating spikes near the edges of the interval, and thus diverging from $f$ as $N$ increases. In order to avoid this pitfall, we suggest controlling the expressiveness of the approximation and  focus on a simple trapezoid-based interpolation.

In the following, we first present the quadrature rule when the quadrature points are evenly spaced in the domain and recall a known result on the upper bound on the approximation error in this trapezoid-based setting. Then, we introduced a refined quadrature rule which selects the quadrature points where the second-order derivative of $f$ varies the most. The main result of this section takes the form of a tighter upper bound on the total approximation error.

\subsection{Uniform approach} \label{sec:uniform}

Let us approximate $f$ on $[z_1, z_2]$, with $z_1, z_2 \in [a,b]$,  by the line passing through $(z_1, f(z_1))$ and $(z_2, f(z_2))$. We set $h=z_2-z_1$. The interpolation $p(x)$ is defined as follows: 
\begin{align}\label{eq:p(x)}
p(x)=\frac{x-z_1}{h}f(z_2)-\frac{x-z_2}{h}f(z_1).
\end{align}
As it is a polynomial of degree $1$, the Lagrange remainder form states that there exists $\xi \in [z_1, z_2]$ such that
\[
f(x)-p(x)=\frac{f''(\xi)}{2}(x-z_1)(x-z_2).
\]
Note that while the existence of $\xi$ is guaranteed, we do not know its position within $[z_1, z_2]$. Moreover, by setting $s=\frac{x-z_1}{h}$, we see that $(x-z_1)(x-z_2)=s(s-1)h^2$. It follows that the error $E_{z_1,z_2}=\int_{z_1}^{z_2}f(x)dx-\int_{z_1}^{z_2}p(x)dx$ on $[z_1, z_2]$ satisfies
\begin{align}\label{eq:error_interval}
E_{z_1,z_2}=-\frac{1}{12}h^3f''(\xi).
\end{align}
We will need the following lemma:

\begin{lemma}
Let $g(x)$ be a continuous function and let $x_0 < x_1 < \dots < x_N$ be points within its domain. Set $c_0, \dots, c_N \geq 0$. Then there exists $\xi \in [x_0,x_N]$ such that
\[
\sum_{i=0}^Nc_ig(x_i)=g(\xi)\sum_{i=0}^Nc_i.
\]
\end{lemma}

\begin{proof}
cf. \cite[Theorem 20.5.1]{Hamming}.
\end{proof}

\begin{proposition}
Set $a,b \in \mathbb{R}$, $N \in \mathbb{N}$, and $f:[a,b]\longrightarrow \mathbb{R}$ a function of class $C^2$, i.e., with  continuous second derivative. Then the total error $E_{\mathrm{tot},\mathrm{unif}}$ of the uniform quadrature of the integral of $f$ by $N$ trapezoids is upper bounded by: 
\begin{eqnarray}
E_{\mathrm{tot},\mathrm{unif}} \leq B_{\mathrm{tot},\mathrm{unif}}=\frac{1}{12}\frac{(b-a)^3}{N^2}\underset{x \in [a,b]}{\on{max}}|f''(x)|. \label{eq:upper1}
\end{eqnarray}
\end{proposition}

\begin{proof}
Consider the interval $[a,b]$ and divide it into $N$ subintervals of length $h=\frac{b-a}{N}$. We set $x_0=a$ and $x_i=x_0+hi$ for $1 \leq i \leq N$. We approximate the function $f$ on each $[x_{i-1}, x_{i}]$ by a trapezoid, and so $p(x)$ is a piece-wise linear function given by by Eq.\eqref{eq:p(x)}. Using Eq.\eqref{eq:error_interval}, for each $1 \leq i \leq N$, there exists $\xi_i\in [x_{i-1}, x_{i}]$ such that the total integration error $E_{\mathrm{tot},\mathrm{unif}}$ on $[x_0, x_N]=[a,b]$ is
\[
E_{\mathrm{tot},\mathrm{unif}}=\sum_{i=1}^N -\frac{1}{12}h^3f''(\xi_i).
\]

By applying the lemma to $E_{\mathrm{tot},\mathrm{unif}}$,  there exists $\xi \in [\xi_1, \xi_N]$ such that
\begin{equation}
E_{\mathrm{tot},\mathrm{unif}}=-\frac{1}{12}Nh^3f''(\xi)=-\frac{1}{12}(b-a)h^2f''(\xi). \hfill \label{eq:E_tot} 
\end{equation}
It follows that for this uniform choice of points $x_0, \dots, x_N$, the total quadrature error satisfies
\begin{align}
\displaystyle |E_{\mathrm{tot},\mathrm{unif}}| \leq B_{\mathrm{tot},\mathrm{unif}}
\end{align}
where $B_{\mathrm{tot},\mathrm{unif}}=\displaystyle \frac{1}{12}\frac{(b-a)^3}{N^2}\underset{x \in [a,b]}{\on{max}}|f''(x)|$.
\end{proof}

\subsection{Second-order Derivative-based Quadrature Rule} \label{sec:refined}

 Rather than defining the quadrature points according to an equispaced uniform grid (as done in the previous section), we suggest here to sample them according to the variations of the second-order derivative of $f$. Consider $f:[a,b] \longrightarrow \mathbb{R}$ a function of class $C^2$ and set $k \leq N$ integers. We divide $[a,b]$ into $k$ intervals $I_j$, $1 \leq j \leq k$, of length $l=\frac{b-a}{k}$. To allow a fair comparison with the uniform method  of Sec.~\ref{sec:uniform}, on each subinterval, we perform a trapezoid interpolation such that the total number of trapezoids is $N$. Let us  split each $I_j$ into $n_j$  subintervals where
\begin{equation}\label{eq:n_j}
   n_j=\biggl\lceil N \frac{\sqrt{M_j}}{ \sum_{p=1}^k \sqrt{M_p}}\biggl\rceil  
\end{equation}
and where $M_j=\underset{x \in I_j}{\on{max}}|f''(x)|$ for each $1 \leq j \leq k$.  
It follows that $\sum_{j=1}^k n_j \approx \sum_{j=1}^k N \frac{\sqrt{M_j}}{ \sum_{p=1}^k \sqrt{M_p}}=N$ where the difference between the number of trapezoids and $N$ is at most $k$ due to the ceiling function. As a consequence, if $N \gg k$, this difference becomes negligible.
For each interval $I_j$, our refined method consists in doing 
a piecewise interpolation of $f_{\,\mkern 1mu \vrule height 2ex\mkern2mu I_j} : I_j \longrightarrow \mathbb{R}$ with $n_j$ trapezoids and then aggregating the results. 
We can now present our main theoretical result.

\begin{theorem}\label{thm:tight}
Set $a,b, \in \mathbb{R}$, $k, N \in \mathbb{N}$ with $k \leq N$, and $f:[a,b]\longrightarrow \mathbb{R}$ a function of class $C^2$. Then the upper bound $B_{\mathrm{tot},\mathrm{refined}} $ on the total error $E_{\mathrm{tot},\mathrm{refined}}$ of our refined quadrature of the integral of $f$ is tighter than that of Eq.\eqref{eq:upper1} with the same number $N$ of trapezoids:
\[
E_{\mathrm{tot},\mathrm{refined}} \leq B_{\mathrm{tot},\mathrm{refined}} = \displaystyle\sum_{j=1}^k \frac{l^3}{12}\frac{1}{\bigg(\biggl\lceil N \frac{\sqrt{M_j}}{ \sum_{p=1}^k \sqrt{M_p}}\biggl\rceil\bigg)^2}|f''(\xi_j)| \leq B_{\mathrm{tot},\mathrm{unif}}.
\]
Here, for each $1 \leq j \leq k$, $\xi_j$ is a well-chosen element in $I_j$. In particular, the more $f''$ varies the more the inequality on the right-hand side is strict, thus in favor of our refined quadrature rule.
\end{theorem}

\begin{proof}
Set $h_j=\frac{l}{n_j}$ for each $1 \leq j \leq k$. Then, for each $1 \leq j \leq k$, we use Eq.\eqref{eq:E_tot} to show that there exists $\xi_j \in I_j$ such that the total error of quadrature satisfies
\[
E_{\mathrm{tot},\mathrm{refined}}=\sum_{j=1}^k -\frac{1}{12}lh_j^2f''(\xi_j).
\]
It follows that $E_{\mathrm{tot},\mathrm{refined}}$ is upper bounded by
{\allowdisplaybreaks
\begin{align}
B_{\mathrm{tot},\mathrm{refined}} & = \ \displaystyle\sum_{j=1}^k \frac{l^3}{12}\frac{1}{\bigg(\biggl\lceil N \frac{\sqrt{M_j}}{ \sum_{p=1}^k \sqrt{M_p}}\biggl\rceil\bigg)^2}|f''(\xi_j)| \nonumber\\[5pt]
 & \ \leq  \displaystyle\sum_{j=1}^k \frac{l^3}{12}\frac{1}{\bigg( N \frac{\sqrt{M_j}}{ \sum_{p=1}^k \sqrt{M_p}}\bigg)^2}|f''(\xi_j)|  \nonumber\\[5pt]
 & \ = \displaystyle \frac{l^3}{12N^2} \sum_{j=1}^k \bigg(\sum_{p=1}^k\frac{\sqrt{M_p}}{\sqrt{M_j}}\bigg)^2|f''(\xi_j)|  \nonumber\\[5pt]
 & \ \leq \displaystyle \frac{l^3}{12N^2} \sum_{j=1}^k \bigg(\sum_{p=1}^k\frac{\sqrt{M}}{\sqrt{M_j}}\bigg)^2|f''(\xi_j)| \hfill \text{ (where } M=\underset{1 \leq p \leq k}{\on{max}}M_p) \nonumber\\[5pt]
& \ =  \displaystyle \frac{l^3}{12N^2} \sum_{j=1}^k k^2 \frac{M}{M_j}|f''(\xi_j)| \nonumber\\[2pt]
 & \ =  \displaystyle \frac{l^3k^2}{12N^2} \bigg(\sum_{j=1}^k  \frac{|f''(\xi_j)|}{M_j}\bigg)M \nonumber\\[2pt]
 & \ \leq  \displaystyle \frac{l^3k^2}{12N^2} \bigg(\sum_{j=1}^k  1\bigg)M  \label{eq:ineq}\\[2pt]
 & \ =  \displaystyle \frac{l^3k^3}{12N^2} M = B_{\mathrm{tot},\mathrm{unif}}. \nonumber 
\end{align}
}

In particular, note that we get Eq.\eqref{eq:ineq} by using the fact that $|f''(\xi_j)| \leq M_j$, and so the more $f''$ varies, the more this inequality is strict.
\end{proof}

\subsection{Experiments}
In this section, we illustrate the behavior of our refined quadrature rule on three functions $f:[a,b]\longrightarrow \mathbb{R}$ and integers $k \leq N$. In order to determine $M_j$ for each $1 \leq j \leq k$, we sample $S=100$ equidistant points $x_{j,s} \in I_j$ and set $M_j=\underset{1 \leq s \leq S}{\on{max}}|f''(x_{j,s})|$. We then compute $n_j$ according to Eq.\eqref{eq:n_j}. It is worth noting at this step of the paper that the cost for computing $M_j$ (here by selecting the max from $S$ evaluations of $f''$) does not matter. The goal of this section is to give evidence that selecting $N+1$ points according to $f''$ is better in terms of quadrature error than using evenly spaced points. When it comes to taking this idea, implementing it in PINNs, and comparing it with SOTA adaptive sampling methods, the same budget in terms of collocation points will be used in the learning process.

In order to avoid pathological cases, we proceed to the following adjustments:
\begin{enumerate}
    \item If $n_j=0$, which theoretically would happen only if $f''=0$ on $I_j$, i.e., $f$ is linear on this interval, then we set $n_j=1$. This will ensure that every interval contributes to the approximation of the integral of $f$.
    \item As explained in the previous section, due to the ceiling function, the sum $\sum_{j=1}^kn_j$ might differ from $N$ by at most $k$. In order to use exactly the same number $N$ of trapezoids, we use the following rule: while $\sum_{j=1}^kn_j \neq N$,  
    \begin{itemize}
        \item[$\bullet$] if $\sum_{j=1}^kn_j >N$, then decrease $\underset{1 \leq j \leq k}{\on{max}}n_j$ by $1$;
        \item[$\bullet$] else increase $\underset{1 \leq j \leq k}{\on{min}}n_j$ by $1$.
    \end{itemize}
\end{enumerate}

For the uniform method, we select $N+1$ equidistant points between $a$ and $b$ included. These are the endpoints of the $N$ trapezoids. If we write $x_i < x_{i+1}$ for the endpoints of such a trapezoid, we approximate the integral of $f$ on $[x_i,x_{i+1}]$ by the area of said trapezoid, i.e., by $(x_{i+1}-x_i)\frac{f(x_{i+1})+f(x_i)}{2}$. It then remains to sum over all trapezoids. {The code of the examples below is available on \href{https://github.com/Antoine-ml-code/Adaptive-Sampling-for-Collocation-Points-in-PINNs-ECML-2025.git}{GitHub}.}

\subsubsection{Example 1:}\label{sec:example_1}

Consider $f(x)= (-1.4+3x^2)\on{sin}(16x)$ on $[0,2]$. For illustration, we choose $N=25$ and $k=11$. We report in Fig.\,\ref{fig:function_1} the target function $f$ (in red)  with the uniform (left) and refined (right) trapezoid approximations (in blue), and the boundary of the intervals (in green). We can see that on the right part of the domain, where there are more variations, our method is able to automatically place more points in this region. On the other hand, it uses  a smaller budget where $f$ varies less. This leads to a better approximation of the integral, reflected by a much smaller relative error ($5.47 \%$) compared to the uniform method ($15.3 \%$).

\begin{figure}[t]
\begin{center}
\includegraphics[width=0.45\textwidth]{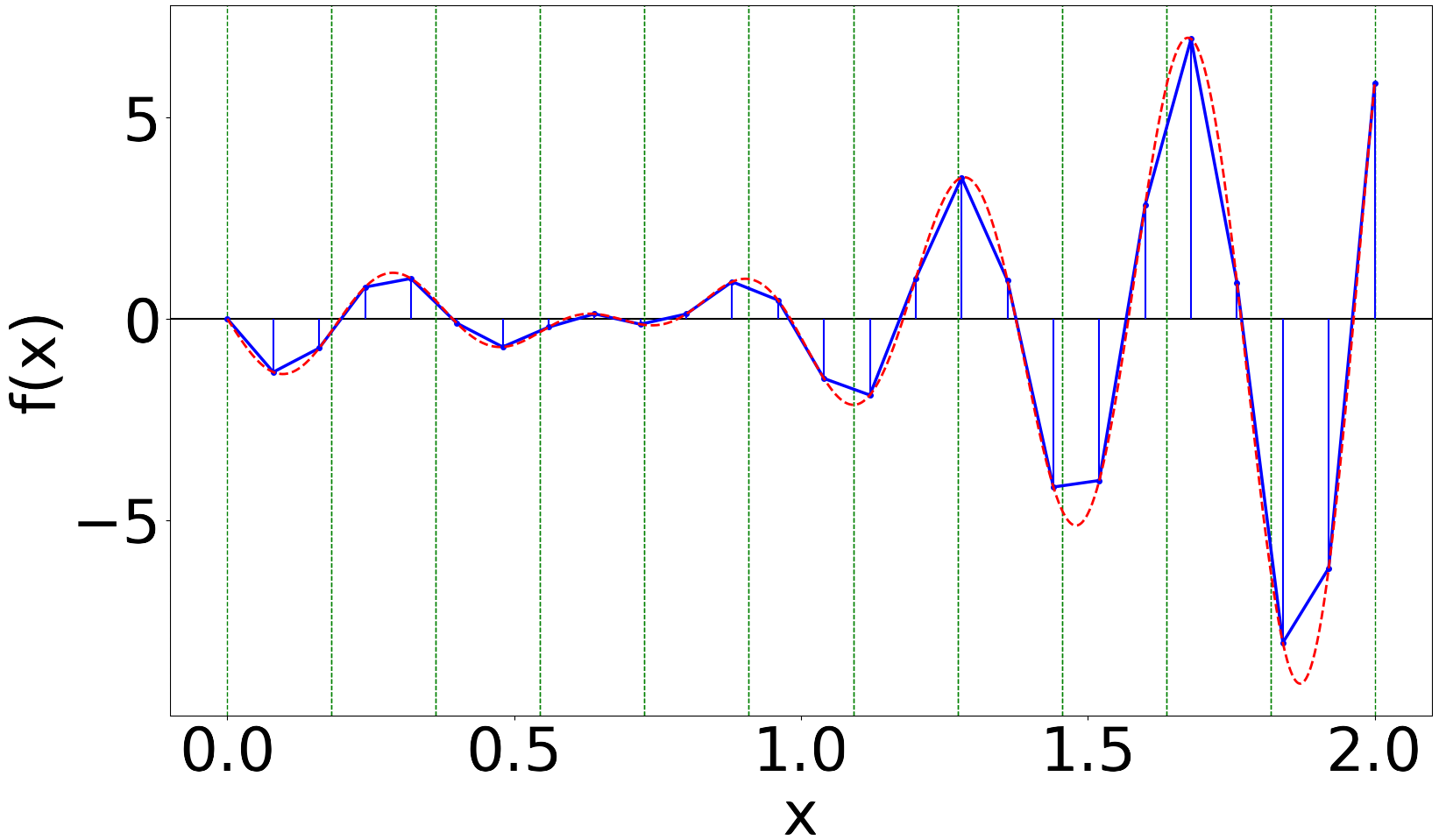}
\quad \quad \includegraphics[width=0.4\textwidth]{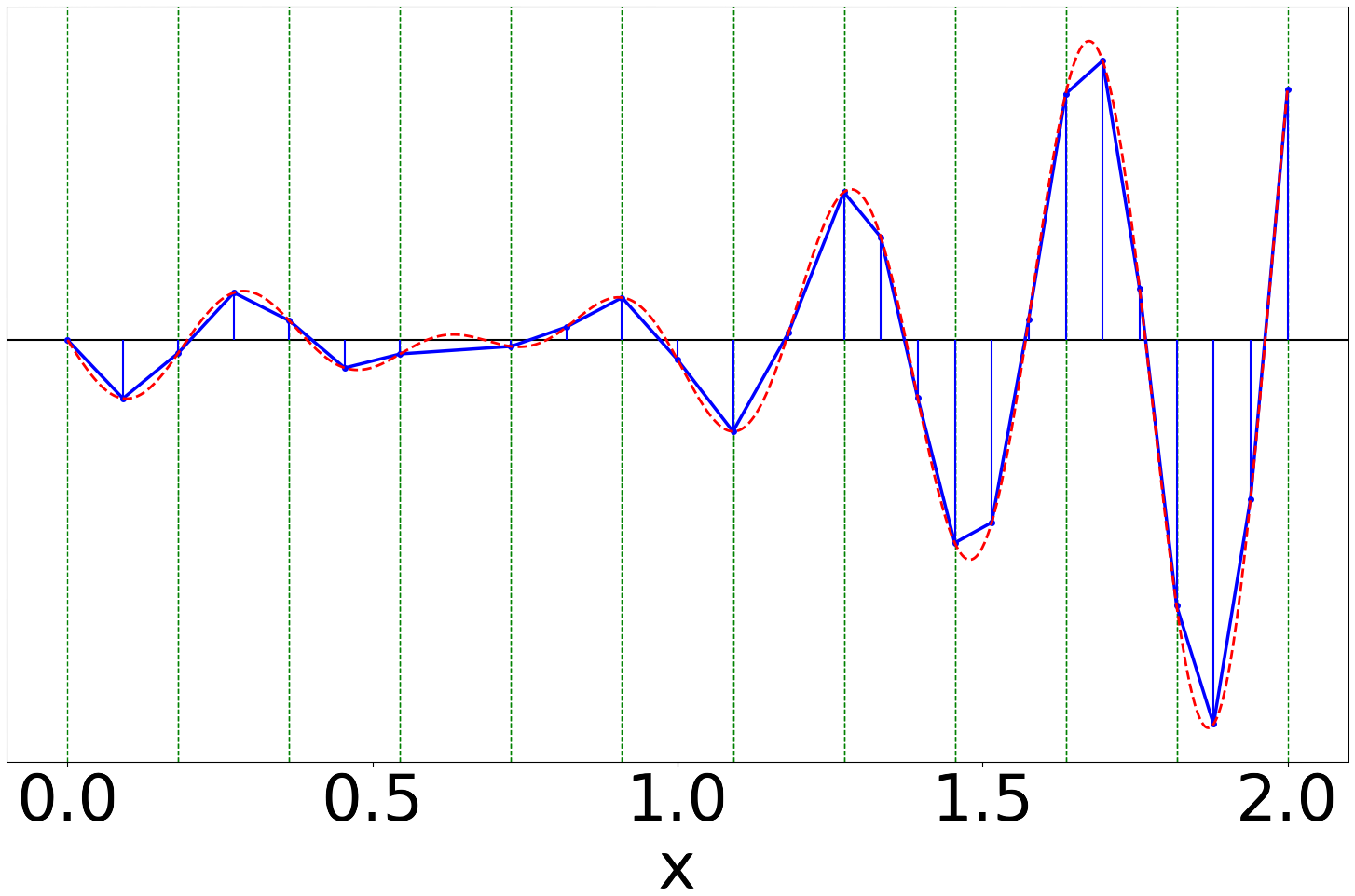}
\caption{Plots of $f(x)= (-1.4+3x^2)\on{sin}(16x)$ on $[0,2]$ (in red) and its trapezoid approximations (in blue) with $N=25$; (left): uniform method; (right): refined method with $k=11$. The relative  errors are respectively $15.3 \%$ and $5.47 \%$.}
\label{fig:function_1}
\end{center}
\end{figure}

We also compare in Fig.\,\ref{fig:Error vs nb trapez_1} both methods by varying the number $k$ of intervals (10, 20, 30 and 40). We see that except for some local spikes, our refined method (black line) consistently gives better results than the uniform method (dashed red line, which is independent from $k$). As expected, the gain is higher as $k$ increases, our method leveraging more precise information about the local variations of the function. Note that the exceptions (i.e., when the black line is above the red one) 
are mainly due to the use of the ceiling function in Eq.\eqref{eq:n_j} which may lead to  $\sum_{j=1}^kn_j \neq N$. In this case, we need to resort to the aforementioned manual adjustments that may lead to local overpessimistic approximations. 

\begin{figure}[h]
\includegraphics[width=1\textwidth]{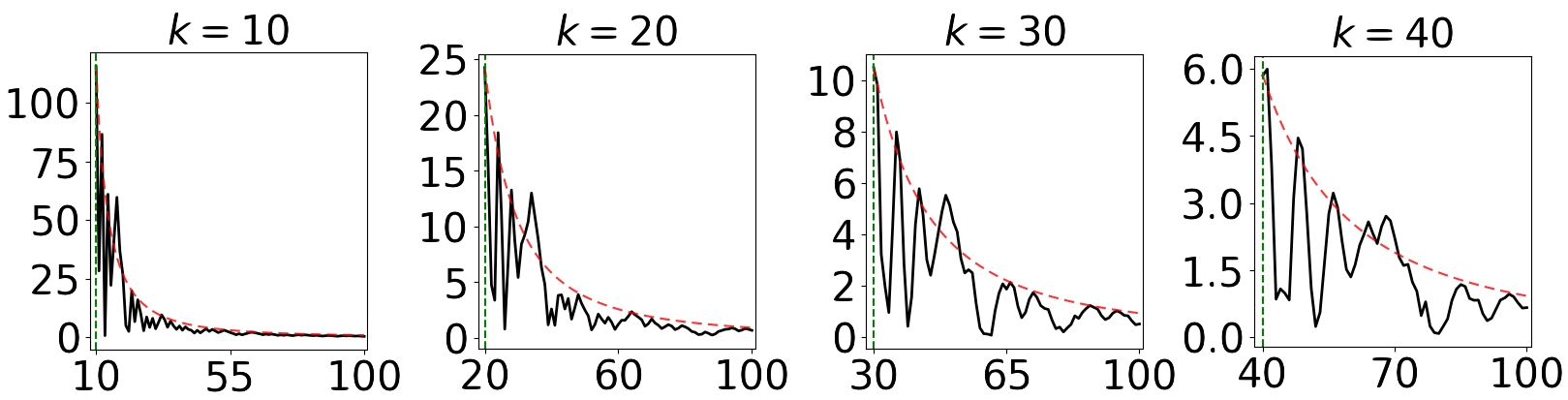}
\caption{Relative error for approximating $f(x)= (-1.4+3x^2)\on{sin}(16x)$ on $[0,2]$ as a function of $N$ for different values of $k$; (red): uniform method; (black): refined quadrature; (green): $\{N=k\}$.}
\label{fig:Error vs nb trapez_1}
\end{figure}

\subsubsection{Example 2:}\label{sec:example_2}
Consider now the function $f(x)=\sin(\frac{1}{\sqrt{x}})$ on $[0.1,1]$. We report in Fig.\,\ref{fig:function_2} an illustration of the behavior of the two quadrature methods when $N=25$ and $k=10$. This example highlights a pathological behavior of the uniform method which is not able with evenly spaced quadrature points to capture the large variations of $f$ (red curve) on some small intervals. On the other hand, our method uses only a little part of the budget (7 points among 25) for approximating the right-hand part of the function and keeps most of these quadrature points for locations where the function has steep changes.

\begin{figure}[t]
\begin{center}
\includegraphics[width=0.45\textwidth]{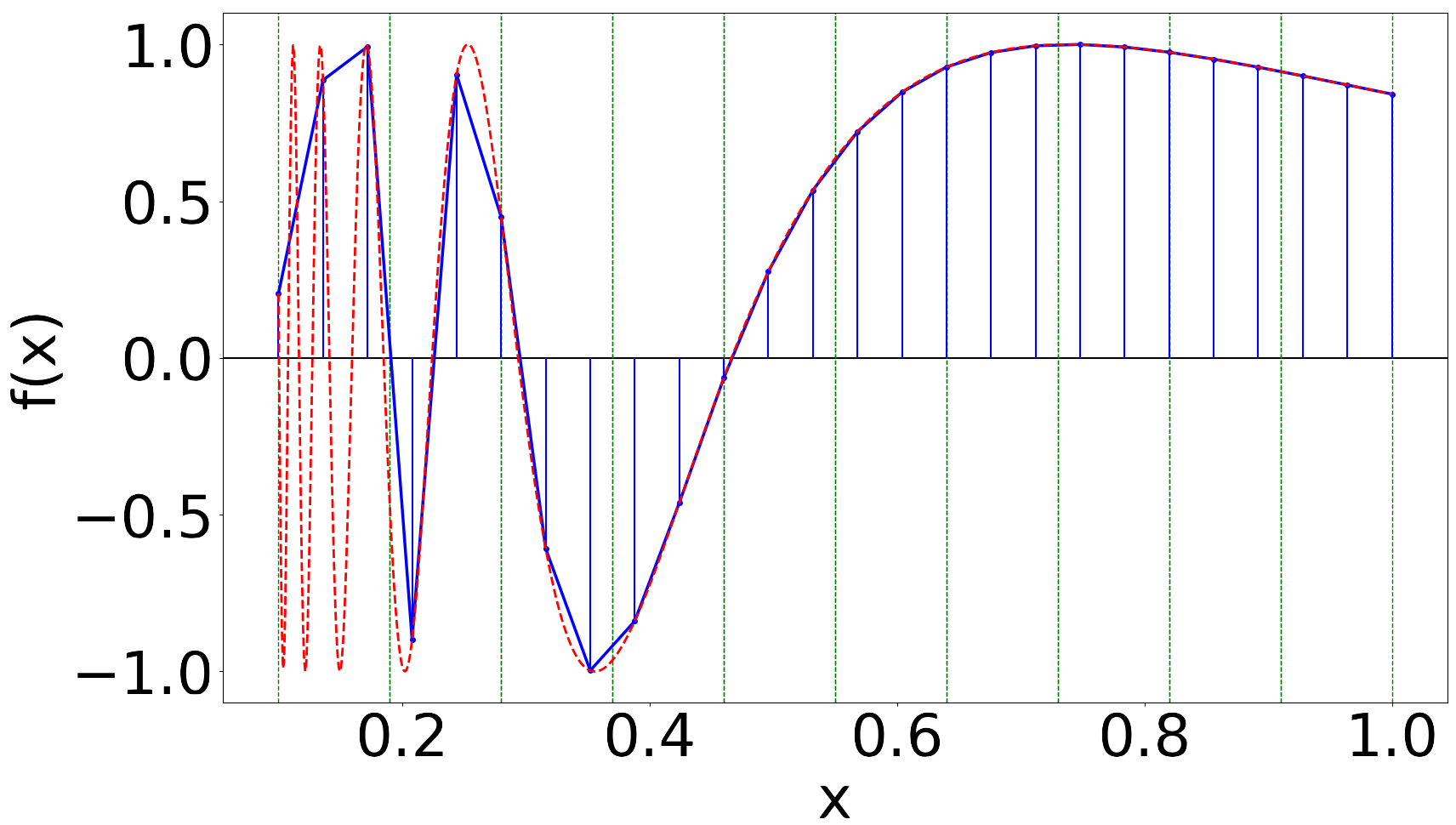}
\quad \quad \includegraphics[width=0.38\textwidth]{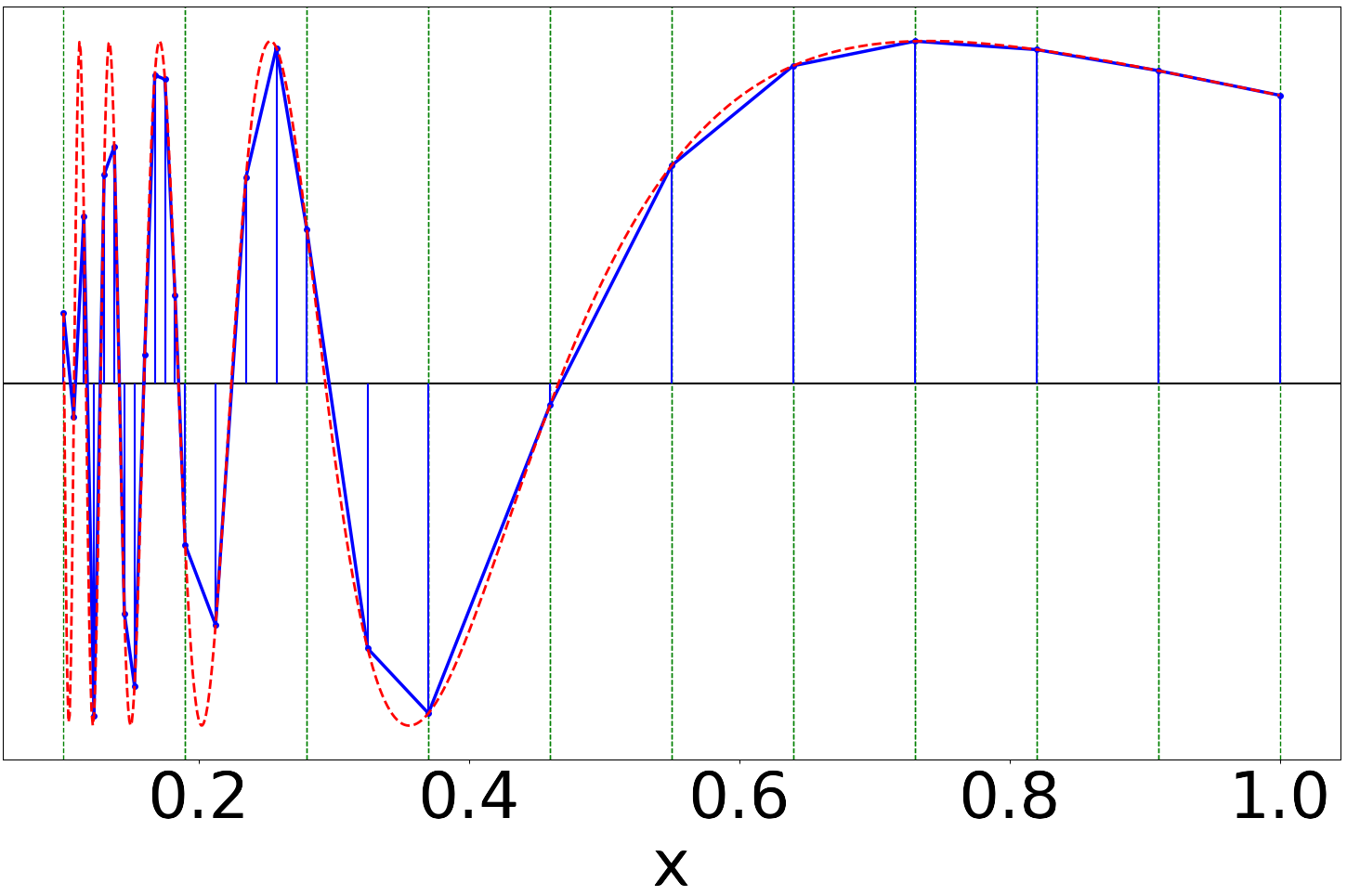}
\caption{Plots of $f(x)=\sin(\frac{1}{\sqrt{x}})$ on $[0.1,1]$ (in red) and its trapezoid approximations (in blue) with $N=25$; (left): uniform method; (right): refined method with $k=10$. The relative  errors are respectively $16.4 \%$ and $1.89 \%$.}
\label{fig:function_2}
\end{center}
\end{figure}

Fig.\,\ref{fig:Error vs nb trapez_2} reports the relative approximation error as $k$ grows from 10 to 40. As already observed in the first example, the relative gain of our refined method increases as $k$ grows by benefiting from finer intervals and thus better capturing the variations of the function.

\begin{figure}[h]
\begin{center}
\includegraphics[width=0.95\textwidth]{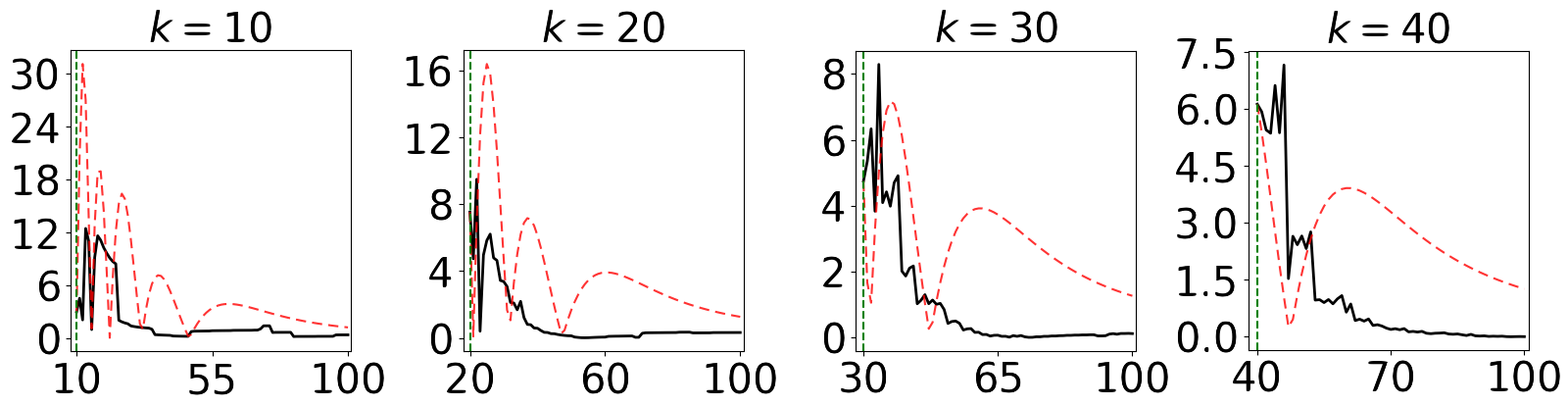}
\caption{Relative error for approximating $f(x)=\sin(\frac{1}{\sqrt{x}})$ on $[0.1,1]$ as a function of $N$ for different $k$; (red): uniform method; (black): refined; (green): $\{N=k\}$.}
\label{fig:Error vs nb trapez_2}
\end{center}
\end{figure}

\subsubsection{Example 3:}\label{sec:example_3}
The last function considered in these experiments is defined as follows on $[0,2]$. It describes a sort of shark fin. 
\begin{align}
f(x)=\left\{\begin{array}{lcl}
-0.1+\sqrt{1.22-(x-1.1)^2} &\text{ if }& 0 \leq x <1, \\[5pt]
1.1-\sqrt{1.22-(x-2.1)^2} & \text{ if }& 1 \leq x \leq 2. \label{eq:shark}
\end{array}\right.
\end{align}

As illustrated in Fig.\,\ref{fig:function_3}, this function has been chosen to emphasize the huge difference between the two methods in terms of density of quadrature points along the domain. Whereas our  quadrature rule concentrates most of the points in regions with  steep variations, the uniform method makes no distinction and wastes part of the budget on easily predictable areas at the expense of a higher approximation error in the more complex regions.

\begin{figure}[t]
\centering{
\includegraphics[width=0.46\textwidth]{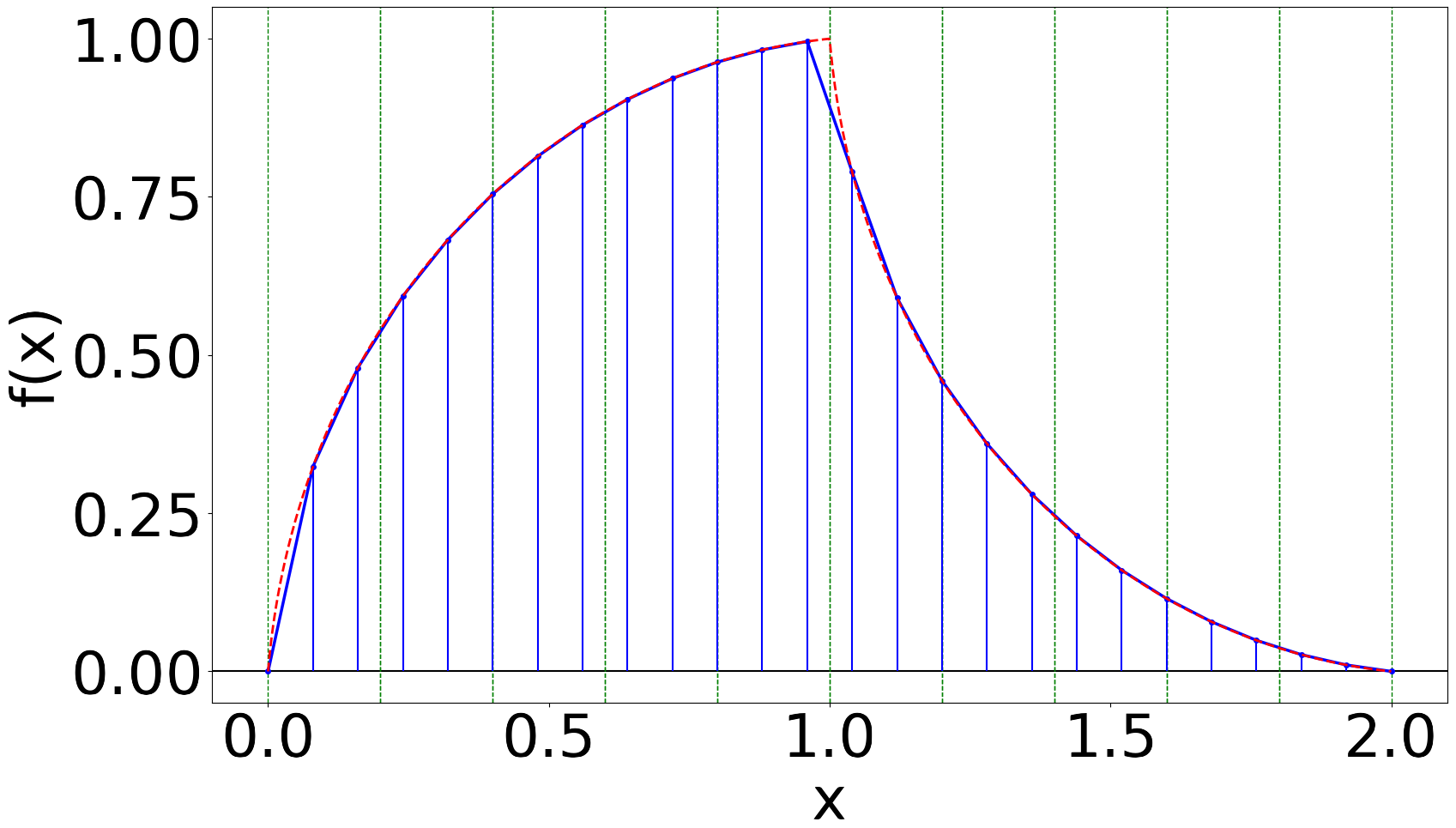}
\quad \quad \includegraphics[width=0.4\textwidth]{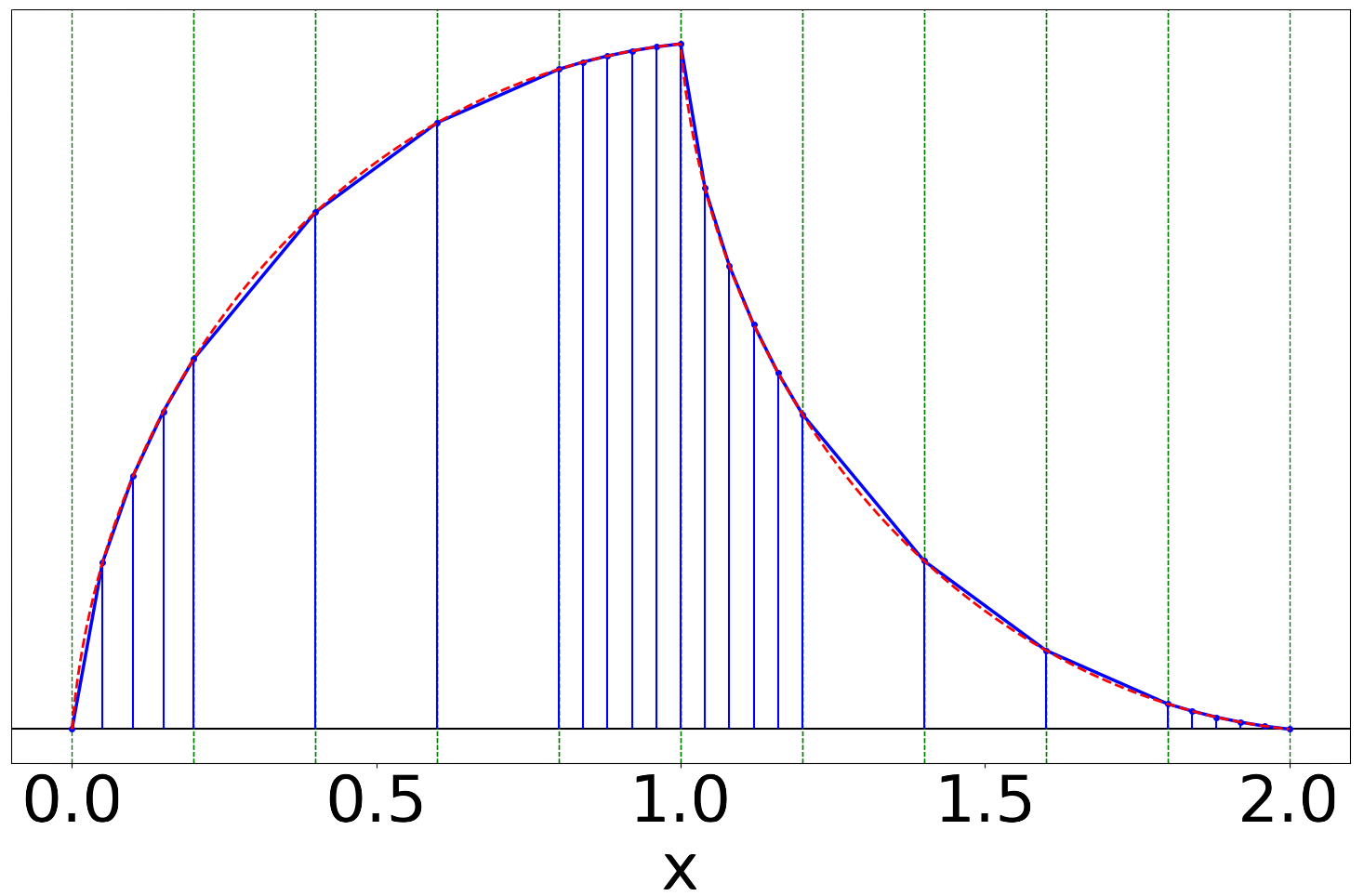}
\caption{Plots of function \eqref{eq:shark} (in red) and its trapezoid approximations (in blue) with $N=25$; (left): uniform method; (right): refined method with $k=10$. The relative quadrature errors are respectively $0.59 \%$ and $0.049 \%$.}
\label{fig:function_3}}
\end{figure}

Fig.\,\ref{fig:Error vs nb trapez_3} confirms the benefit of our method (black line) in terms of relative error for different values of $k$. In particular, we notice oscillations with the uniform method based on the parity of $N$. This is because when $N$ is odd, the top vertex of the curve is not a vertex of a trapezoid, so the precision decreases. 

\begin{figure}[h]
\includegraphics[width=1\textwidth]{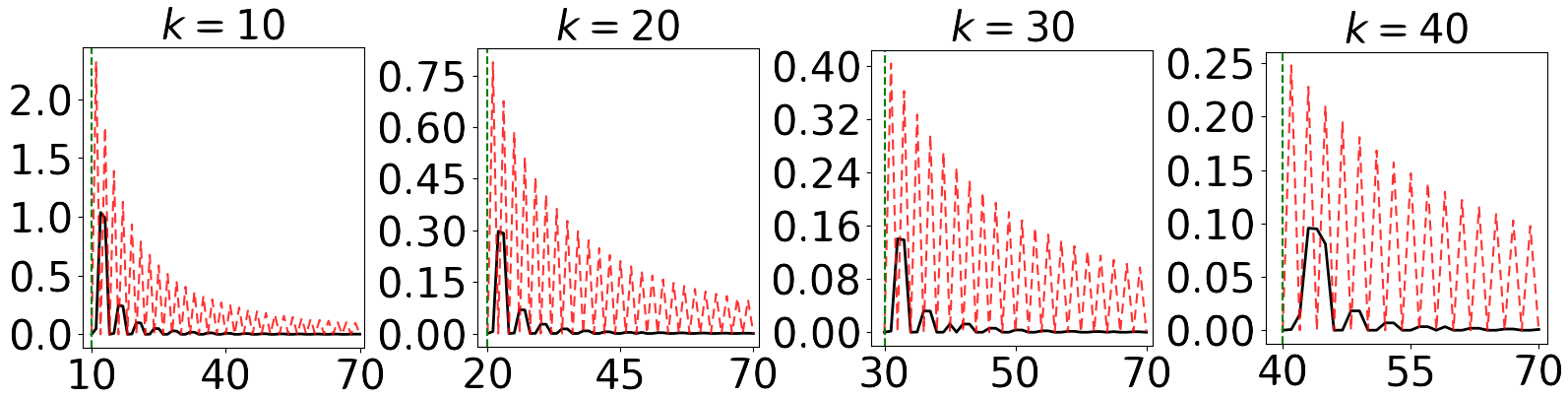}
\caption{Relative error for approximating Eq.~\eqref{eq:shark} on $[0,2]$ as a function of $N$ for different values of $k$; (red): uniform method; (black): refined; (green): $\{N=k\}$.}
\label{fig:Error vs nb trapez_3}
\end{figure}

\section{Adaptive Sampling Methods for PINNs} \label{sec:PINN}

The previous theoretical and experimental results highlighted the importance of using the second-order derivative of $f$ in a quadrature rule for better approximating its integral. In the context of PINNs, 
where $f$ is the integrand of the loss function made of the PDE residuals (i.e., the integrand of ${\mathcal L}(\theta,\mathbf{x})$), our results  state that using information from the Hessian of the residuals for sampling collocation points is better than uniformly sampling them. 
As mentioned in the related work section, while leveraging the gradient of the residuals has been recently used in a couple of papers (see, e.g., \cite{subramanian2023,visser2024,gPINN2021}), as far as we know, resorting to the Hessian has not been investigated yet. Even though we formally proved that using $f''$ to define quadrature points is better than selecting evenly spaced points, we do not yet know how such a strategy behaves in PINNs when compared to sampling methods that leverage the magnitude or the gradient of the residuals. This is the goal of this section, which aims to gain a comprehensive grasp of the capabilities and limitation of $f''$ on different PDEs. 

\subsection{Generic Algorithm STAR-RAD ($\bigstar$-RAD)}
To allow a fair comparison, we use a RAD-like framework  where the $N$ collocation points are sampled according to a probability density function proportional to {\it a criterion of interest}. The latter can be the {\bf PDE residuals} as used in RAD \cite{Chenxi2022}, the {\bf gradient} of the residuals as in \cite{subramanian2023}, the {\bf Hessian} of the residuals for our method, or a {\bf uniform distribution} as used in a standard PINN \cite{raissi2019physics}. In order to use the same setting for this comparison study, we rely on the RAD algorithm and modify it so as to allow different underlying probability density functions. Let us use the following generic distribution:
\begin{eqnarray}
    p(x) & \propto & \frac{\gamma(\mathbf{x})^{\tau}}{\mathbb{E}[\gamma(\mathbf{x})^{\tau}]}+c, \label{eq:PDF}
\end{eqnarray}
where $\tau$ and $c$ are hyperparameters. This formulation is interesting because the SOTA sampling methods can be viewed as special cases of Eq.\eqref{eq:PDF}. Let us consider  them as instantiations of what we call in the following $\bigstar$-RAD, where  res-RAD, grad-RAD, hessian-RAD, and unif-RAD correspond respectively to the residual-based (i.e., where $\gamma(\mathbf{x})=f(\mathbf{x})$), gradient-based ($\gamma(\mathbf{x})=f'(\mathbf{x})$), Hessian-based ($\gamma(\mathbf{x})=f''(\mathbf{x})$) and uniform-based sampling method (standard PINN obtained with $\tau=0$ and $c \rightarrow \infty$). The pseudo-code of $\bigstar$-RAD is presented in Algorithm~\ref{algo:*RAD}.

\begin{algorithm}[t]
\caption{$\bigstar$-RAD} \label{algo:*RAD}
\begin{algorithmic}[1]
\State Set $\bigstar$ $\in \{ ``\text{res}", ``\text{grad}", ``\text{hessian}", ``\text{unif}"\}$, $\tau$, $c$, $N$ and $\#epochs$.
\State Sample a set $S$ of initial collocation points randomly.
\State Train a PINN for a certain number of epochs.
\While{$\#epochs$ not reached}
    \State Build distribution $p(x)$ of Eq.\eqref{eq:PDF} given $\bigstar$ from a set of random points.
    \State $S \leftarrow$ New set of $N$ collocation points sampled according to $p(x)$.
    \State Train a PINN for a certain number of epochs.
\EndWhile
\end{algorithmic}
\end{algorithm}

\subsection{Experimental Results}
Here we perform experiments for three different PDEs, namely the 1D Newton's law of cooling, 1D Brinkman-Forchheimer equation, and 2D Poisson's equation. The analysis is mainly made in terms of convergence speed of the methods, keeping in mind that faster convergence can also be interpreted as a lower need for collocation points to achieve the same performance after a certain number of epochs. The experiments\footnote{The code is available in our \href{https://github.com/Antoine-ml-code/Adaptive-Sampling-for-Collocation-Points-in-PINNs-ECML-2025.git}{GitHub repository.}} have been conducted using ADAM optimizer on a Apple M1 Pro chip with 16Go RAM.

\subsubsection{Newton's law of cooling equation:}
Newton's law of cooling  describes the rate of heat loss of a body as follows:
$\frac{d T}{d t}=R(T_{env}-T(t)),$ 
with $t \in [0, 1000]$, $T_{env}=25$, $T(0)=100$, and $R=0.005$ (coefficient of heat transfer). The analytical solution given by $T(t)=T_{env}+(T(0)-T_{env})e^{-Rt}$ states that this rate is proportional to the difference in the temperatures between the body and its environment. 
We learn a  PINN with a RELU activation function composed of 4 hidden layers with 100 neurons followed by a fully connected layer. The number of epochs $\#epochs=30000$, the learning rate $\eta=10^{-5}$, the number of collocation points $N=40$ drawn according to $p(x)$ (Eq.\eqref{eq:PDF}) approximated from $4000$ candidates, $\tau=1/2$  and $c=0$. We resample every 1000 epochs.

\begin{figure}[t]
{\centering
\includegraphics[width=0.42\textwidth]{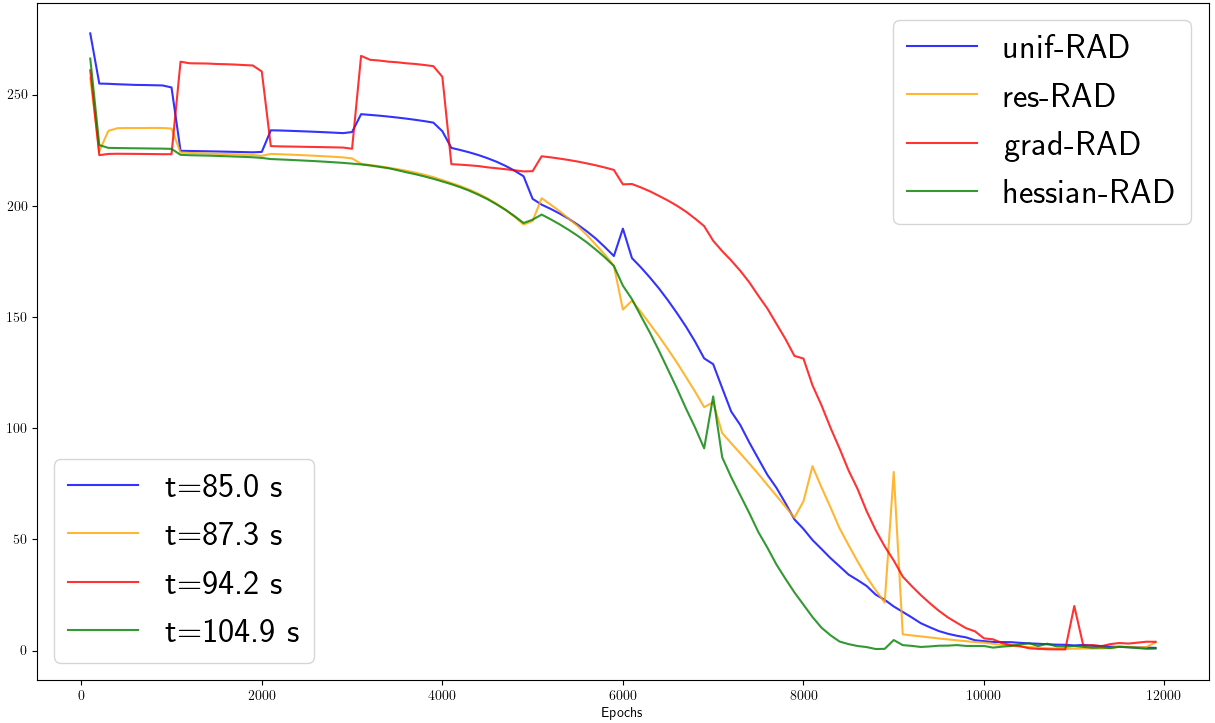}
\includegraphics[width=0.42\textwidth]{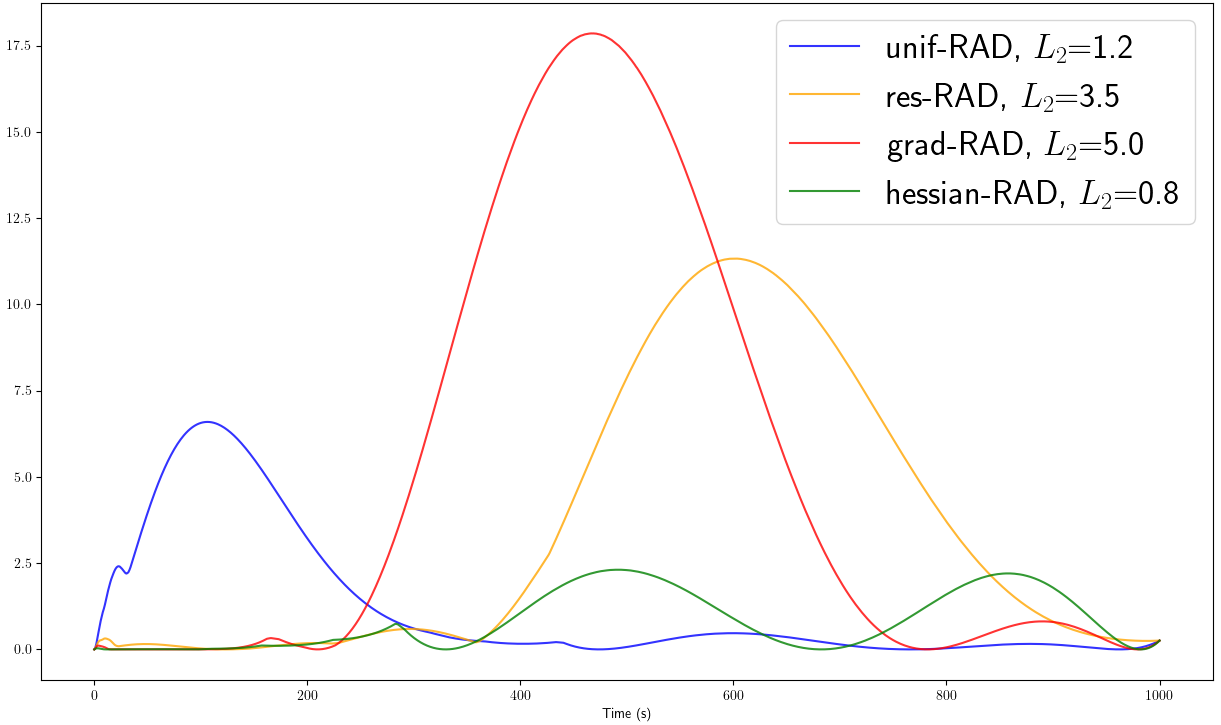}
\caption{(Left) $L_2$-test error along the first 12000 training epochs, as well as the total computational time (after the total 30000 epochs), for various sampling methods on Newton's law of cooling; (Right)  Squared prediction errors along the  domain $[0, 1000]$ after 12000 epochs as well as the total $L_2$-error.}
\label{fig:newton}}
\end{figure}

Fig.\,\ref{fig:newton} (left) reports the $L_2$-test error (i.e., $(T(x)-u_{\theta}(x))^2$) computed along the first 12000 training epochs (i.e., until convergence is reached for all methods) from an equispaced uniform grid composed of 1000 test points. We can see that even though the four competing methods successfully learn the neural solver, benefiting from the second-order derivative (green curve) allows to converge faster. To illustrate the gain in terms of prediction errors over the entire domain, Fig.\,\ref{fig:newton} (right) describes the error suffered by $u_{\theta}(x)$ after 12000 epochs (the behaviors of the 4 methods do not change afterwards). We can see that our hessian-RAD gives a  better approximation of the  solution without suffering from a too large computational burden. 

\subsubsection{Brinkman-Forchheimer:}
The Brinkman–Forchheimer model is a extended Darcy's law and is used to describe wall-bounded porous media flows: 
$$\displaystyle -\frac{\nu_e}{\epsilon}\frac{d^2u}{d \mathbf{x}^2}+\frac{\nu}{K} u(\mathbf{x})=g,$$
with $\mathbf{x} \in [0,H]$, $\nu_e=\nu=10^{-3}$, $\epsilon=0.4$, $K=10^{-3}$, $g=1$, and $H=1$. The analytical solution is $u(x)=\displaystyle \frac{gK}{\nu}\left(1-\frac{\on{cosh}(r(x-\frac{H}{2})}{\on{cosh}(\frac{rH}{2})} \right)$ with $r=\displaystyle \sqrt{\frac{\nu \epsilon}{\nu_e K}}$ and is 
depicted in Fig.\,\ref{fig:Brinkman} (right, black curve). $u$ represents the fluid velocity, $g$ denotes the external force, $\nu$ is the kinetic viscosity of the fluid, $\epsilon$ is the
porosity of the porous medium, and $K$ is the permeability. The effective viscosity $\nu_e$ is related to the pore structure. A no-slip boundary condition is imposed, i.e., $u(0)=u(1)=0$. We learn a PINN with the tanh activation function composed of 3 hidden layers with 20 neurons followed by a fully connected layer. We used the following parameters: $\#epochs=30000$, $\eta=10^{-3}$, $N=30$, $\tau=1/2$  and $c=0$. We resample every 1000 epochs. 

\begin{figure}[t]
{\centering
\includegraphics[width=0.42\textwidth]{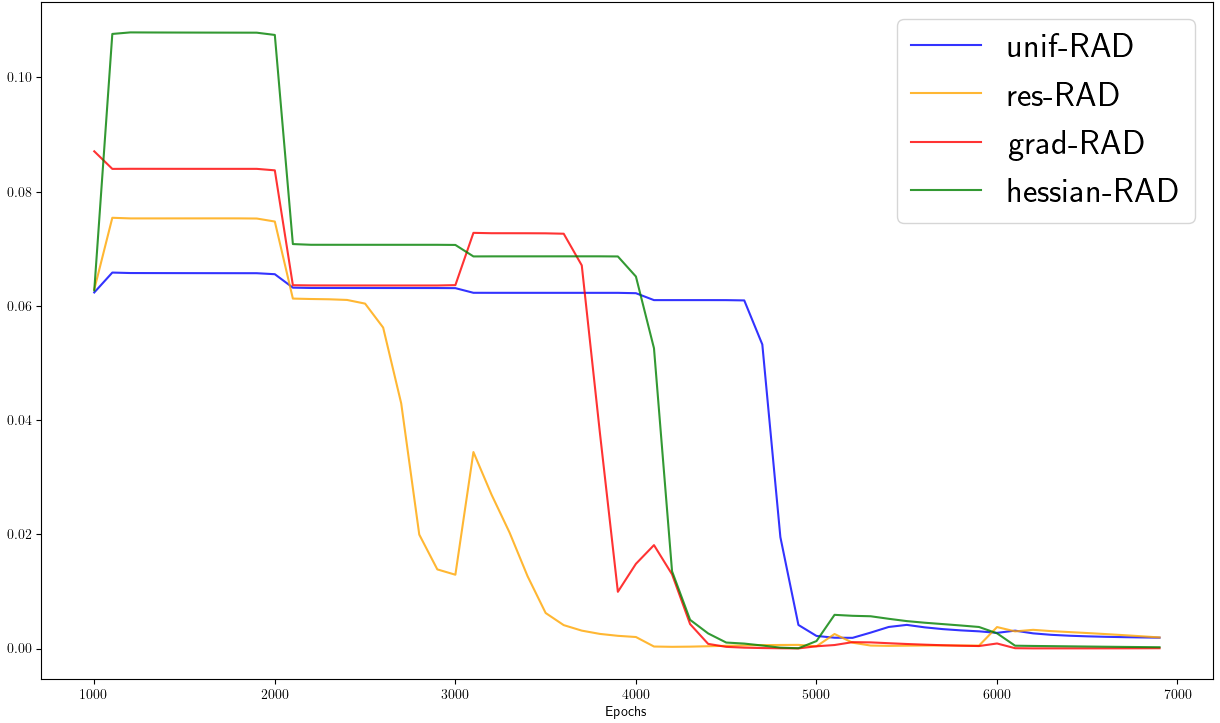}
\includegraphics[width=0.42\textwidth]{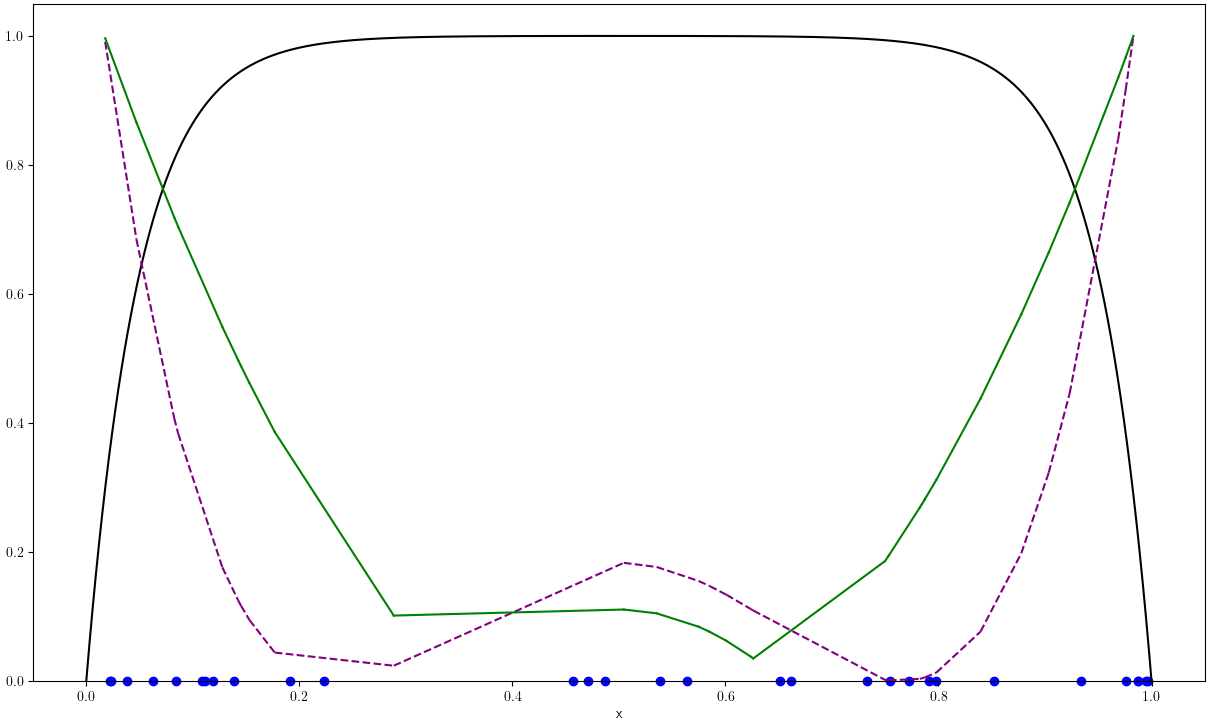}
\caption{(Left) Comparison of the $L_2$-test errors between iterations 1000 and 7000 on Brinkman-Forchheimer; (Right): Analytical solution of the PDE (black), normalized loss (purple dashed) and $|f''|$ (green) after 3000 epochs, and collocation points (in blue) generated by hessian-RAD after 3000 epochs.}
\label{fig:Brinkman}}
\end{figure}

Fig.\,\ref{fig:Brinkman} (left) reports the $L_2$-test error computed along the first 7000 training epochs before convergence of the 4 competing methods. If we can observe that the three adaptive methods (using $f$, $f'$ and $f''$) are better than a standard uniform sampling-based PINN (blue line), this figure also states that the convergence of derivative-based methods (both $f'$ and $f''$) is a bit slower than a residual-based sampling. The reason for this phenomenon comes from the shape of the function which, apart the initial and final steep changes, presents a large plateau. To analyze the impact of the latter, we plot on Fig.\,\ref{fig:Brinkman} (right) the residuals (dashed purple line) as well as $|f''|$ (green line) after 3000 epochs (illustrating a situation where $f$ is much better than $f''$). As expected, as $f''$ does not vary much between $0.3$ and $0.6$, hessian-RAD places only a few  collocation points along this interval, keeping most of the budget where it varies the most. Consequently, the resulting PINN makes errors in this region that do not affect the empirical loss too much, but leading to a poor behavior at test time. The same interpretation can be provided for grad-RAD, both methods requiring more iterations to converge. Nevertheless, note  that hessian-RAD reaches eventually the best prediction. 

\subsubsection{2D Poisson's PDE:} Poisson's equation is an elliptic PDE used in theoretical physics. It involves second derivatives of $u(x,y)$ and is given by $
\Delta u=F(x,y),$ where $(x,y) \in [0,1]^2$. We take $F$ such that $u(x,y)=2^{4a}x^a(1-x)^ay^a(1-y)^a$ with $a=10$ is the analytical solution (depicted in Fig.\,\ref{fig:Poisson} (top left)).
We learn a PINN with the tanh activation function composed of 3 hidden layers with 20 neurons followed by a fully connected layer. We used the following parameters: $\#epochs=20000$, $\eta=10^{-3}$, $N=400$ drawn according to $p(x)$  approximated from $40000$ candidates, $\tau=1/2$  and $c=0$. We resample every 1000 epochs. 
\begin{figure}[t]
\includegraphics[width=\textwidth]{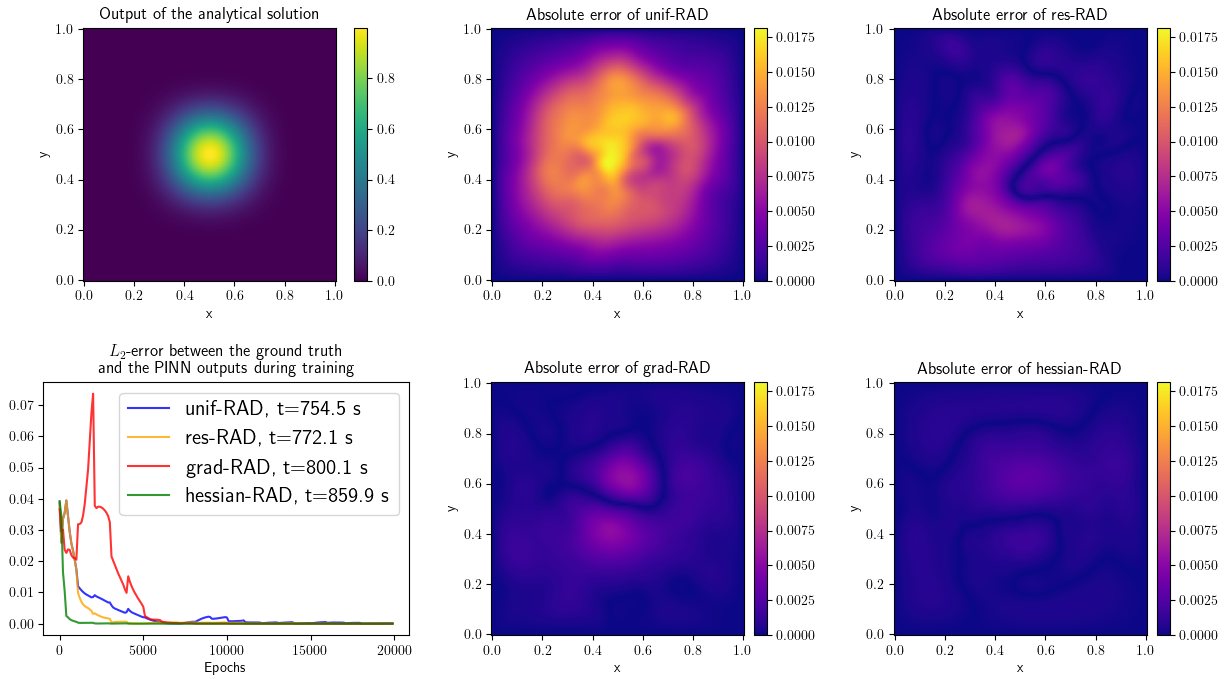}
\caption{(Top left) Analytical solution of Poisson's PDE; (Bottom Left) $L_2$-test error along the first 20000 training epochs, with computational time; (Right) Heatmaps of errors of the 4 sampling methods after 20000 epochs.}
\label{fig:Poisson}
\end{figure}
The most striking comment we can make from Fig.\,\ref{fig:Poisson} (bottom left) is that hessian-RAD fully takes advantage of the abrupt variations of the Poisson solution to converge much faster than the others. About 1000 epochs are sufficient for learning  the problem while the competing strategies require much more iterations to stabilize.  Interestingly, even after 20000 epochs when the methods seem to have converged to an exact solution, the gap in terms of prediction error in favor of our method is important, as illustrated with the four heat maps of Fig.\,\ref{fig:Poisson} (right part). Again, even though computing the Hessian is more costly, the additional burden is reasonable and compensated by a better prediction.


\section{Conclusion and Perspectives}
We have presented a provably accurate quadrature method based on second-order derivatives, which performs very well for estimating the integral of a function $f$. Exploiting the Hessian of the residuals shows also promising results when used in a sampling method for PINNs. The observations made from the Brinkman-Forchheimer PDE give raise to a future possible direction consisting in sampling the collocation points according to different distributions. This has already been done in \cite{subramanian2023}, but only by combining uniformly sampled points with others drawn according to residual information (magnitude or gradient). The results obtained on this PDE rather suggest that a combination of $f$ with $f'$ and/or $f''$ in complicated regions would be relevant. Identifying automatically these challenging parts of the domain is an open question. One could consider a hybrid approach in which the zones where the derivative is below some threshold are decomposed into grids. The size of such a grid could be proportional to the size of the zone and the total number of collocation points. Elsewhere, the sampling would follow the values of the derivative. This would ensure that such areas are not left out during the training process and the $L_2$-test error might go down more rapidly. 
On the other hand, while using the Hessian has been shown to be beneficial in a PINN training, it can become costly in high dimensions. 
A first mitigation attempt would be a simple stochastic approach, where at each resampling iteration, the entries of the Hessian to be computed are sampled.
Another approach is to build on the fact that methods like gPINN are beneficial and already do a big part of the computations necessary for the Hessian, and thus these can be combined almost for free.
Indeed, these methods compute the gradient (w.r.t. the parameters) of the gradient (w.r.t. spatio-temporal dimensions).
In a deep network, this gradient of gradient already needs to backpropagate almost back to the input and thus computing the Hessian is almost free in such a case. This would also allow having a resampling step that is executed more often, possibly at every iteration, by computing the empirical max of $f''$ (using the current collocation points) in cells and resampling points using this information. Finally, note that the reasoning in Sec.~\ref{sec:uniform} is based on the Lagrange remainder theorem, which itself uses the fact that $p(x)$ is a polynomial in $x$ of degree at most $1$, hence $p''(x)=0$. Consider $f(x,y)$ of two variables in a square $[x_1,x_2]\times [y_1,y_2]$. The natural analogue of $p$ would be a function $P(x,y)=a_0+a_1x+a_2y+a_3xy$ linear in both $x$ and $y$ such that $f=P$ on the corners of the square. But if $a_3 \neq 0$, then $P(x,y)$ has non-zero second order terms, so the reasoning cannot be extended. Moreover, if one were to push the computations further, the error term becomes too complicated to manipulate like in the proof of Th.~\ref{thm:tight}. Another approach is hence needed in order to generalize to higher dimensions.

\medbreak

\noindent \textbf{Acknowledgments.}
This work has been funded by a public grant from the French National Research Agency under the “France 2030” investment plan, which has the reference EUR MANUTECH SLEIGHT - ANR-17-EURE-0026.


\delete{
\medbreak

\noindent \textbf{Disclosure of Interests.} The authors have no competing interests to declare that are relevant to the content of this article.

\paragraph{Impact Statement}
This essentially fundamental and theoretical paper aims at advancing the field of PiML for PINNs. Our work may have many societal or ethical implications, but we feel that none need be specifically highlighted here.
}

\end{document}